\documentclass{article}

\usepackage{times}
\usepackage{graphicx} 
\usepackage{subfigure} 

\usepackage{dsfont}
\usepackage{amsmath}
\usepackage{amsfonts}
\usepackage{amssymb}
\usepackage{amsthm}
\usepackage{subfigure}
\usepackage{epsfig}
\usepackage{color}
\usepackage{enumerate}
\usepackage{placeins}
\usepackage{natbib}

\usepackage{algorithm}
\usepackage{algorithmic}

\usepackage{hyperref}

\usepackage[accepted]{icml2017}

\newtheorem{lemma}{Lemma}
\newtheorem{theorem}{Theorem}
\newtheorem{corollary}{Corollary}
\newtheorem{definition}{Definition}
\newtheorem{assumption}{Assumption}
\newtheorem{remark}{Remark}

\theoremstyle{definition}

\icmltitlerunning{Density Level Set Estimation on Manifolds with DBSCAN}

\begin{document} 

\twocolumn[
\icmltitle{Density Level Set Estimation on Manifolds with DBSCAN}


\begin{icmlauthorlist}
\icmlauthor{Heinrich Jiang}{a}
\end{icmlauthorlist}
\icmlaffiliation{a}{Google}
\icmlcorrespondingauthor{Heinrich Jiang}{heinrich.jiang@gmail.com}

\icmlkeywords{}

\vskip 0.3in
]
\printAffiliationsAndNotice{}

\begin{abstract} 
We show that DBSCAN can estimate the connected
components of the $\lambda$-density level set $\{ x : f(x) \ge \lambda\}$ given $n$ i.i.d. samples from an unknown density $f$.
We characterize the regularity of the level set boundaries using parameter $\beta > 0$ and analyze the estimation error under the Hausdorff metric.
When the data lies in $\mathbb{R}^D$ we obtain a rate of $\widetilde{O}(n^{-1/(2\beta + D)})$, which matches known lower bounds
up to logarithmic factors. When the data lies on an embedded unknown $d$-dimensional manifold in $\mathbb{R}^D$, then we obtain a rate of
$\widetilde{O}(n^{-1/(2\beta + d\cdot \max\{1, \beta \})})$.
Finally, we provide adaptive parameter tuning in order to attain these rates with no a priori knowledge
of the intrinsic dimension, density, or $\beta$.
\end{abstract} 

\section{Introduction}

DBSCAN \cite{DBSCAN} is one of the most popular clustering algorithms amongst practitioners and has 
had profound success in a wide range of data analysis applications. 
However, despite this, its statistical properties have not been fully understood.
The goal of this work is to give a theoretical analysis of the procedure 
and to the best of our knowledge, provide the first analysis of density level-set estimation
on manifolds. We also contribute ideas to related areas that may be of independent interest.

DBSCAN aims at discovering clusters which turn out to be the high-density regions of the 
dataset. It takes in two hyperparameters: minPts and $\varepsilon$. It defines a point as a 
\textit{core-point} if there are at least minPts sample points in its $\varepsilon$-radius neighborhood. 
The points within the $\varepsilon$-radius neighborhood of a core-point are said to be
 $\textit{directly reachable}$ from that core-point. Then, a point $q$ is $\textit{reachable}$ from a core-point $p$ if there
 exists a path from $q$ to $p$ where each point is directly reachable from the next point. It is now clear that 
 this definition of reachable gives a partitioning of the dataset (and remaining points not reachable from any core-point are considered noise).
 This partitioning is the clustering that is returned by DBSCAN. 

The problem of analyzing DBSCAN has recently been explored in \cite{dbscanConsistency}.
Their analysis is for a modified version of DBSCAN and is not focused on estimating a fixed density level. 
Their results have many desirable properties,
but are not immediately applicable for what this paper tries to address. Using recent developments in topological data analysis along with some tools we
develop in this paper, we show that
it is now possible to analyze the original procedure.

The clusters DBSCAN aims at discovering can be viewed as approximations of the connected components of 
the level sets $\{x : f(x) \ge \lambda \}$ where $f$ is the density and $\lambda$ is some density level.
We provide the first comprehensive analysis in tuning minPts and $\varepsilon$
to estimate the density level set for a particular level. Here, the density level $\lambda$ is known to the algorithm while the density remains unknown.
Density level set estimation has been studied
extensively. e.g., \citep{carmichael,hartigan, polonik95, cuevas, walther97, tysbakovMinimax, baillo, cadre, willet, biau, RV09, maier09, adaptive, generalizedDensity, S11, rinaldo12, S15, chen16, jiang17}.
However approaches that obtain state-of-art consistency results are largely unpractical (i.e. unimplementable).
Our work shows that in actuality, DBSCAN, a procedure known for decades and has since been used widely,
can achieve the strongest known results. Also, unlike much of the existing work,
we show that DBSCAN can also recover the connected components of the level sets separately and bijectively.

Our work begins with the insight that DBSCAN behaves like an $\varepsilon$-neighborhood graph, 
which is different from, but related to the $k$-nearest neighbor graph. The latter
has been heavily used for cluster-tree estimation \cite{CD10, stuetzle10, KV11, CDKvL14, jiang2017modal} and in this paper we adapt some of 
these ideas for $\varepsilon$-neighborhood graphs.

Cluster-tree estimation aims at discovering the hierarchical tree structure of the connected-components as the levels vary.
\citet{balakrishnan2013cluster} extends results by \citet{CD10} to the setting where the data lies
on a lower dimensional manifold and provide consistency results depending on the lower dimension and
independent of the ambient dimension. Here we are instead interested in how to set minPts and $\varepsilon$ in order to estimate a particular 
level and provide rates on the Hausdorff distance error. This is different from works on cluster tree estimation which
focuses on how to recover the tree structure rather than recovering a particular level.
In that regard, we also require density estimation bounds in order to get a handle on the true density-levels and the empirical ones. 

\citet{optimalknn} gives us optimal high-probability finite-sample $k$-NN density estimation bounds which hold {\it uniformly}; this is key to
obtaining optimal level-set estimation rates under the Hausdorff error. Much of the previous works on density level-set estimation, e.g. \citep{RV09} provide rates under risk measures such as symmetric set difference. These metrics are considerably weaker than the Hausdorff metric; the latter is a uniform guarantee. There are such bounds for the histogram density estimator.
This allowed \citet{adaptive} to obtain optimal rates under Hausdorff metric, while 
having a fully adaptive procedure. This was a significant breakthrough for level set estimation, as discussed by \citet{chazal15}.
We believe this to be the strongest consistency results obtained thus far. However, a downside is that
the histogram density estimator has little practical value. 
Here, aided with the desired bounds on the $k$-NN density estimator, we can actually obtain similar results to \citet{adaptive} but with the clearly practical DBSCAN.

We extend the $k$-NN density estimation results of \citet{optimalknn} to the manifold case, as the bulk our analysis is about the more general case that
the data lies on a manifold. Density-based procedures perform
poorly in high-dimensions since the number of samples required increases exponentially in the dimension-- the so called curse of dimensionality.
Thus, the consequences of handling the manifold case are of practical significance. Since the estimation rates we obtain depend only on 
the intrinsic dimension, it explains why DBSCAN can do well in high dimensions if the
data has low intrinsic dimension (i.e. the manifold hypothesis). Given the modern capacity of systems to collect data of increasing complexity,
it has become ever more important to understand the feasibility of {\it practical} algorithms in high dimensions.

To analyze DBSCAN, we write minPts and $\varepsilon$ in terms of the $d$, unknown manfold dimension; $k$, which controls the density estimator; and $\lambda$, which determines which level to estimate. We assume knowledge of $\lambda$ with the goal of estimating the $\lambda$-level set of the density. We give a range of $k$ in terms of $n$ and corresponding consistency guarantees and estimation rates for such choices. 
We then adaptively tune $d$ and $k$ in order to attain close to optimal performance with no a priori knowledge of the distribution.
Adaptivity is highly desirable because it allows for automatic tuning of the hyper-parameters, which is a core tenet of unsupervised learning.
To solve for the unknown dimension, we use an estimator from \citet{intrinsicKnnNew}, which we show to have considerably better finite-sample behavior than previously thought. More details and discussion of related works is in the main text. We then provide a new method of choosing $k$ such that it will asymptotically approach a value that provides near-optimal level set estimation rates.

\section{Overview}
We start by analyzing the procedure under the manifold assumption. The end of the paper will discuss the full-dimensional setting.
The bulk of our contribution lies in analyzing the former situation, while the analysis of the latter uses a subset of those techniques.
\vspace{-0.3cm}
\begin{itemize}
	\setlength\itemsep{0.0em}
\item Section~\ref{dbscan-knn} proves that the clusters returned by DBSCAN are close to the connected components
of certain $\varepsilon$-neighborhood graphs (Lemma~\ref{dbscan}). This is significant because these graphs can be shown
to estimate density level sets.

\item Section~\ref{manifoldsetting-section} introduces the manifold setting and provides supporting results including
$k$-nearest neighbor density estimation bounds
(Lemma~\ref{fk_upper_bound} and Lemma~\ref{fk_lower_bound}) that are useful later on.

\item Section~\ref{consistency-section} provides a range of parameter settings under which for each true cluster, there exists a
corresponding cluster
returned by DBSCAN (Lemma~\ref{separation} and Lemma~\ref{connectedness}), and a rate for the Hausdorff distance
between them (Theorem~\ref{hausdorfferror}). 

\item Section~\ref{pruning-section} shows how one can apply DBSCAN a second time to remove false clusters 
from the first application, thus completing a bijection between the estimates and the true clusters (Theorem~\ref{pruning}).

\item Section~\ref{adaptive-section} explains how to adaptively tune the parameters so that they fall within the theoretical ranges.
The main contributions of this section are a stronger
result about a known $k$-nearest neighbor based approach to estimating the unknown dimension (Theorem~\ref{estimating_d})
and a new way to tune $k$ to approach an optimal choice of $k$ (Theorem~\ref{estimating_alpha}).

\item Section~\ref{fulldimensionalsetting-section} gives the result
when the data lives in $\mathbb{R}^D$ without the manifold assumption.

\end{itemize}

\section{The connection to neighborhood graphs} \label{dbscan-knn}

This section is dedicated towards the understanding of the clusters produced by DBSCAN.
The algorithm can be found in~\cite{DBSCAN} and is not shown here since Lemma~\ref{lemma:dbscan} characterizes what 
DBSCAN returns.

We have $n$ i.i.d. samples $X = \{x_1,...,x_n\}$ drawn from a distribution $\mathcal{F}$ over $\mathbb{R}^D$.
\begin{definition}
\label{def:knn}
Define the $k$-NN radius of $x\in \mathbb{R}^D$ as 
\begin{align*}
r_k(x) := \inf\{r > 0: |X \cap B(x, r)| \ge k \},
\end{align*} where $B(x, r)$ denotes the Euclidean ball of radius $r$ centered at $x$.
 Let $G(k, \varepsilon)$ denote the $\varepsilon$-neighborhood level graph of $X$ with vertices $\{ x \in X : r_k(x) \le \varepsilon \}$ and an edge between $x$ and $x'$ iff $||x - x'|| \le \varepsilon$. 
 \end{definition}
 \begin{remark}
This is slightly different from $\varepsilon$-neighborhood graph, which includes all vertices. Here we exclude vertices below certain empirical density level (i.e. $r_k(x) > \varepsilon$).
 \end{remark}
 The next definition is relevant to DBSCAN and is from \cite{DBSCAN} but in the notation of Definition~\ref{def:knn}.
 \begin{definition} The following is with respect to fixed $\varepsilon > 0$ and $\text{minPts} \in \mathbb{N}$. 
 \vspace{-0.3cm}
 \begin{itemize}
 	\setlength\itemsep{0.0em}
 \item $p$ is a core-point if $r_{\text{minPts}}(p) \le \varepsilon$.
 \item $q$ is directly density-reachable from $p$ if $|p-q| \le \varepsilon$ and $p$ is a core-point.
 \item $q$ is density-reachable from $p$ if there exists a sequence $q = p_1,p_2,...,p_m = p$ such that $p_{i}$ is directly density-reachable from $p_{i+1}$ for $i = 1,..,m-1$.
\end{itemize}
 \end{definition}

 The following result is paraphrased from Lemmas 1 and 2 from \cite{DBSCAN}, which characterizes the clusters learned by DBSCAN. 
 \begin{lemma} \cite{DBSCAN} \label{lemma:dbscan}
Let $\mathcal{C}$ be the clusters returned by DBSCAN(minPts, $\varepsilon$). For any core-point $x$, there exists $C \in \mathcal{C}$ with $x \in C$. On the other hand, for any $C \in \mathcal{C}$, there exists core-point $x$ such that $C = \{x' : \text{$x'$ is density-reachable from $x$} \}$.
 \end{lemma}

 We now show the following result relating the $\varepsilon$-neighborhood level graphs and the clusters obtained from DBSCAN. Such an interpretation of DBSCAN has been given in previous works such as \citet{campello2015hierarchical}.

  \begin{lemma} [DBSCAN and $\varepsilon$-neighborhood level graphs] \label{dbscan}
 Let $\mathcal{C}$ be the clusters obtained from $\text{DBSCAN}(minPts,\varepsilon)$ on $X$.
 Let $\mathcal{K}$ be the connected components of $G(minPts, \varepsilon)$. 
 Then, there exists a one-to-one correspence between $\mathcal{C}$ and $\mathcal{K}$ such that if $C \in \mathcal{C}$ and $K \in \mathcal{K}$ correspond, then
 \vspace{-0.2cm}
 \begin{align*}
 K \subseteq C \subseteq \cup_{x \in K} B(x, \varepsilon) \cap X.
 \end{align*}
 \end{lemma}
\vspace{-0.3cm}
 \begin{proof}
Take any $K \in \mathcal{K}$. Each point in $K$ is a core-point and by Lemma~\ref{lemma:dbscan} and the definition of density-reachable, each point in $K$ belongs to the same $C \in \mathcal{C}$. Thus, $K \subseteq \{ x \in C : r_k(x) \le \varepsilon\}$. Next we show that $K = \{ x \in C : r_k(x) \le \varepsilon\}$.

Suppose there exists core-point $x\in C$ but $x \notin K$ and let $y \in K$. By Lemma~\ref{lemma:dbscan}, there exists core-point $c \in C$ such that all points in $C$ are directly reachable from $c$. Then there exists a path of core-points from $x$ to $c$  with pairwise edges of length at most $\varepsilon$. The same holds for $c$ to $y$. Thus there exists such a path of core-points from $x$ to $y$, which means that $x, y$ are in the same CC of $G(minPts, \varepsilon)$, contradicting the assumption that $x \notin K$ and $y \in K$. Thus, in fact $K = \{ x \in C : r_k(x) \le \varepsilon\}$. The result now follows since $C$ consists of points that are at most $\varepsilon$ from its core-points. 
\end{proof}
\vspace{-0.3cm}

We can now see that DBSCAN's clusterings can be viewed as the connected components (CCs) of an appropriate $\epsilon$-neighborhood level graph. 
Using a neighborhood graph to approximate the level-set has been studied in \cite{generalizedDensity}. The difference is 
that they use a kernel density estimator instead of a $k$-NN density estimator and study the convergence properties under different settings.

\section{Manifold Setting} \label{manifoldsetting-section}

\subsection{Setup}
We make the following regularity assumptions which are standard among works 
on manifold learning e.g. \citep{manifold07, manifold12, balakrishnan2013cluster}.
\begin{assumption}\label{manifold}
$\mathcal{F}$ is supported on $M$ where:
\vspace{-0.3cm}
\begin{itemize}
\setlength\itemsep{0.0em}
\item $M$ is a $d$-dimensional smooth compact Riemannian manifold without boundary embedded in compact subset $\mathcal{X} \subseteq \mathbb{R}^D$.
\item The volume of $M$ is bounded above by a constant.
\item $M$ has condition number $1/\tau$, which controls the curvature and prevents self-intersection.
\end{itemize}
\vspace{-0.3cm}
Let $f$ be the density of $\mathcal{F}$ with respect to the uniform measure on $M$.
\end{assumption}
\begin{assumption}\label{density}
$f$ is continuous and bounded.
\end{assumption}

\subsection{Basic Supporting Bounds}

The following result bounds the empirical mass of Euclidean balls to the true mass under $f$. 
It is a direct consequence of Lemma 6
of \citet{balakrishnan2013cluster}.

\begin{lemma} [Uniform convergence of empirical Euclidean balls (Lemma 6 of \citet{balakrishnan2013cluster})] \label{ballbounds}
Let $\mathcal{N}$ be a minimal fixed set such that each point in $M$ is at most distance $1/n$ from some point in $\mathcal{N}$.
There exists a universal constant $C_0$ such that the following holds with probability at
least $1 - \delta$.
For all $x \in X \cup \mathcal{N}$,
\begin{align*}
\mathcal{F}(B) \ge C_{\delta, n} \frac{\sqrt{d \log n}}{n} &\Rightarrow \mathcal{F}_n(B) > 0\\
\mathcal{F}(B) \ge \frac{k}{n} + C_{\delta, n} \frac{\sqrt{k}}{n} &\Rightarrow \mathcal{F}_n(B) \ge \frac{k}{n} \\
\mathcal{F}(B) \le \frac{k}{n} - C_{\delta, n}\frac{\sqrt{k}}{n} &\Rightarrow \mathcal{F}_n(B) < \frac{k}{n}.
\end{align*}
where $C_{\delta, n} = C_0 \log(2/\delta) \sqrt{d \log n}$, $\mathcal{F}_n$ is the empirical distribution, and $k \ge C_{\delta, n}$.
\end{lemma}
\begin{remark}
For the rest of the paper, many results are qualified to hold with probability at least $1 - \delta$. This is 
precisely the event in which Lemma~\ref{ballbounds} holds.
\end{remark}
\begin{remark}
If $\delta = 1/n$, then $C_{\delta, n} = O((\log n)^{3/2})$.
\end{remark}

Next, we need the following bound on the volume of the intersection Euclidean ball and $M$; this is required to get a handle on the true mass of the ball under
$\mathcal{F}$ in later arguments. The upper and lower bounds follow from \citet{upperBoundBall} and Lemma 5.3 of \citet{lowerBoundBall}. The proof is given in the appendix. 

\begin{lemma} [Ball Volume] \label{ballvolume}
If $0 < r < \min\{\tau/4d, 1/\tau\}$, and $x \in M$ then
\begin{align*}
v_d r^d (1 - \tau^2 r^2) \le \text{vol}_d(B(x, r) \cap M) \le v_d r^d (1 + 4dr/\tau).
\end{align*}
where $v_d$ is the volume of a unit ball in $\mathbb{R}^d$ and $\text{vol}_d$ is the volume w.r.t. the uniform measure on $M$. 
\end{lemma}

\subsection{$k$-NN Density Estimation}

Here, we establish density estimation rates for the $k$-NN density estimator
in the manifold setting. This builds on work in density estimation on manifolds e.g. \citep{hendriks90,pelletier05,ozakin09,kim13,berry17}; thus,
it may be of independent interest.
The estimator is defined as follows

\begin{definition} [k-NN Density Estimator] \label{kNNdensity}
\begin{align*}
f_k(x) := \frac{k}{n\cdot v_d\cdot r_k(x)^d}.
\end{align*}
\end{definition}

The following extends previous work of \citet{optimalknn} to the manifold case.
The proofs can be found in the appendix.

\begin{lemma}[$f_k$ upper bound]\label{fk_upper_bound} Suppose that Assumptions~\ref{manifold} and~\ref{density} hold.
Define the following which charaterizes how much the density increases locally in $M$:
\begin{align*}
\hat{r}(\epsilon, x) &:=\sup\left\{r : \sup_{x' \in B(x, r) \cap M} f(x') - f(x) \le \epsilon \right\}.
\end{align*}
Fix $\lambda_0 > 0$ and $\delta > 0$ and suppose that $k \ge C_{\delta, n}^2$.
Then there exists constant $C_1 \equiv C_1 (\lambda_0, d, \tau)$ such that if 
\begin{align*}
k \le C_1 \cdot C_{\delta, n}^{2d/(2+d)} \cdot n^{2/(2+d)},
\end{align*}
then the following holds with probability at least $1 - \delta$
uniformly in $\epsilon > 0$ and $x \in X$ with $f(x) + \epsilon \ge \lambda_0$:
\begin{align*}
f_k(x) < \left(1 + 3 \cdot \frac{C_{\delta, n}}{\sqrt{k}} \right)\cdot (f(x) + \epsilon),
\end{align*}
provided $k$ satisfies $v_d\cdot \hat{r}(\epsilon, x)^d \cdot (f(x) + \epsilon) \ge \frac{k}{n} - C_{\delta, n}\frac{\sqrt{k}}{n}$. 
\end{lemma}

\begin{lemma}[$f_k$ lower bound] \label{fk_lower_bound} Suppose that Assumptions~\ref{manifold} and~\ref{density} hold.
Define the following which charaterizes how much the density decreases locally in $M$:
\begin{align*}
\check{r}(\epsilon, x) &:=\sup\left\{r : \sup_{x' \in B(x, r) \cap M} f(x) - f(x') \le \epsilon \right\}.
\end{align*}
Fix $\lambda_0 > 0$ and $0 < \delta < 1$ and suppose $k \ge C_{\delta, n}$.
Then there exists constant $C_2 \equiv C_2(\lambda_0, d, \tau)$ such that if
\begin{align*}
k \le C_2 \cdot C_{\delta, n}^{2d/(4 + d)}\cdot n^{4 / (4 + d)},
\end{align*}
then with probability at least $1-\delta$, the following holds uniformly
for all $\epsilon > 0$ and $x \in X$ with $f(x) - \epsilon \ge \lambda_0$:
\begin{align*}
f_k(x) \ge \left(1 - 3 \cdot \frac{C_{\delta, n}}{\sqrt{k}} \right)\cdot (f(x) - \epsilon),
\end{align*}
provided $k$ satisfies $v_d\cdot \check{r}(\epsilon, x)^d \cdot (f(x) - \epsilon) \ge \frac{4}{3}\left(\frac{k}{n} + C_{\delta, n}\frac{\sqrt{k}}{n} \right)$. 
\end{lemma}

\begin{remark}
We will often bound the density of points with low density. In low-density regions,
there is less data and thus we require more points to get a tight bound. However, in many cases
a tight bound is not necessary; thus the purposes of $\epsilon$ is to allow some slack.
The higher the $\epsilon$, the easier it is for the lemma conditions to be satisified. 
In particular, if $f$ is $\alpha$-H\"older continuous (i.e. $|f(x) - f(x')| \le C_\alpha |x - x'|^\alpha$), we have $\hat{r}(\epsilon, x), \check{r}(\epsilon, x) \ge (\epsilon / C_{\alpha})^{1/\alpha}$.
\end{remark}

\section{Consistency and Rates} \label{consistency-section}

\subsection{Level-Set Conditions}

Much of the results will depend on the behavior of level set boundaries. Thus, we require sufficient drop-off at the boundaries, as well as separation between the CCs at a particular level set. We give the following notion of separation.
\begin{definition}
$A, A'$ are $r$-separated in $M$ if there exists a set $S$ such that every path from $A$ to $A'$ intersects $S$ and 
$\sup_{x \in M \cap (S + B(0, r))} f(x) < \inf_{x \in A \cup A'} f(x)$.
\end{definition}
Define the following shorthands for distance from a point to a set, the intersection of $M$ with a neighborhood around a set under the Euclidean distance,
and the largest Euclidean distance from a point in a set to its closest sample point. 
\begin{definition}
$d(x, A) := \inf_{x' \in A} |x - x'|$,
$C^{\oplus r} := \{x \in M : d(x, C) \le r \}$,
$r_n(C) := \sup_{c \in C} d(c, X)$.
\end{definition}
We have the following mild assumptions which ensures that the CCs can be separated from the rest of the density
by sufficiently wide valleys and there is sufficient decay around the level set boundaries.

\begin{assumption} [Separation Conditions] \label{clusterProperties}
Let $\lambda > 0$ and $\mathcal{C}_\lambda$ be a CCs of $\{ x \in M : f(x) \ge \lambda \}$. 
There exists $\check{C}_\beta, \hat{C}_\beta, \beta, r_s, r_c > 0$ and $0 < \lambda_0 < \lambda$ such that the following holds:

For each $C \in \mathcal{C}_\lambda$, there exists $A_C$, a connected component of $M^{\lambda_0} := \{ x \in M : f(x) \ge \lambda_0 \}$ 
such that:
\vspace{-0.2cm}
\begin{itemize}
		\setlength\itemsep{0.0em}
	\item $A_C$ separates $C$ by a valley: $A_C$ does not intersect with any other CC in $\mathcal{C}_\lambda$;
	$A_C$ and $M^{\lambda_0}  \backslash A_C$ are $r_s$-separated by some $S_C$.
	\item $C^{\oplus r_c} \subseteq A_C$.
	\item $\beta$-regularity: For $x \in C^{\oplus r_c} \backslash C$, we have
	\begin{align*}
	 \check{C}_{\beta} \cdot d(x, C)^{\beta} \le \lambda - f(x) \le \hat{C}_{\beta} \cdot d(x, C)^{\beta}.
	\end{align*}
\end{itemize}
\end{assumption}

\begin{remark}
We can choose any $0 < \beta < \infty$. The $\beta$-regularity assumption appears in e.g. \cite{adaptive}. This is very general and
also allows us to make a separate global 
smoothness assumption.
\end{remark}

\begin{remark}
We currently characterize the smoothness w.r.t. the Euclidean distance. One could alternatively use the geodesic distance 
on $M$, $d_M(p, q)$. It follows from Proposition 6.3 of \citet{lowerBoundBall} that when $|p - q| < \tau/4$, we have $|p - q| \le d_M(p ,q) \le 2 |p - q|$. Since the distances we deal in our analysis with are of such small order,
these distances can thus essentially be treated as equivalent. We use the Euclidean distance throughout the paper for simplicity.
\end{remark}

\begin{remark}
For the rest of this paper, it will be understood that Assumptions~\ref{manifold},~\ref{density}, and~\ref{clusterProperties} hold.
\end{remark}

We can define a region which isolates $C$ away from other clusters of $\{ x \in M : f(x) \ge \lambda \}$.
\begin{definition}
 $\mathcal{X}_{C} := \{ x : \exists \text{ a path }\mathcal{P} \text{ from } x \text{ to } x'\in C \text{ such that } \mathcal{P} \cap S_C = \emptyset\}$. 
\end{definition} 

\subsection{Parameter Settings}\label{parameters}
Fix $\lambda > 0$ and $\delta > 0$. Let $k$ satisfy the following
\begin{align*}
K_l \cdot (\log n)^2 \le k \le K_u \cdot (\log n)^{2d/(2+d)} \cdot n^{2\beta' / (2\beta' + d)},
\end{align*}
where $\beta' := \min\{1, \beta \}$, and $K_l$ and $K_u$ are positive constants depending on $\delta, \check{C}_\beta, \hat{C}_\beta, \beta, \tau, d, ||f||_\infty, \lambda_0, r_s, r_c$ 
which are implicit in the proofs later in this section.

The remainder of this section will be to show that DBSCAN(minPts, $\varepsilon$) with 
\begin{align*}
&\text{minPts} = k, \text{ } \varepsilon = \left(\frac{k}{n\cdot v_d \cdot (\lambda - \lambda \cdot C_{\delta, n}^2/\sqrt{k})} \right)^{1/d}
\end{align*}
will consistently estimate each CC of 
$\{ x \in M : f(x) \ge \lambda \}$.
Throughout the text, we denote $\widehat{\mathcal{C}_\lambda}$ as the clusters returned by DBSCAN under this setting.

\subsection{Separation and Connectedness}
Take $C \in \mathcal{C}_\lambda$. We show that DBSCAN will return an estimated CC
$\widehat{C}$, such that $\widehat{C}$ does not contain any points outside of $\mathcal{X}_C$. 
Then, we show that $\widehat{C}$ contains all the sample points in $C$.
The proof ideas used are similar to that of standard results in cluster trees estimation; they can be found in the appendix.

\begin{lemma} [Separation] \label{separation}
There exists $K_l$ sufficiently large and $K_u$ sufficiently small such that the following holds with probability at least $1 - \delta$. Let $C \in \mathcal{C}_\lambda$. There exists $\widehat{C} \in \widehat{\mathcal{C}_\lambda}$ such that
$\widehat{C} \subseteq \mathcal{X}_C$.
\end{lemma} 

\begin{lemma} [Connectedness]\label{connectedness}
 There exists $K_l$ sufficiently large and $K_u$ sufficiently small such that the following holds with probability at least $1 - \delta$. Let $C \in \mathcal{C}_\lambda$. If there exists $\widehat{C} \in \widehat{\mathcal{C}_\lambda}$ such that
$\widehat{C} \subseteq \mathcal{X}_C$, then $C^{\oplus r_n(C)} \cap X \subseteq \widehat{C}$.
\end{lemma}
\begin{remark}
These results allow $C$ to have any dimension between $0$ to $d$ since we reason with $C^{\oplus r_n(C)}$, which contains samples, instead of simply $C$.
\end{remark}

\subsection{Hausdorff Error}
We give the estimation rate under the Hausdorff metric. 
\begin{definition} [Hausdorff Distance]
\begin{align*}
d_{\text{Haus}}(A, A') = \max \{ \sup_{x \in A} d(x, A'), \sup_{x' \in A'} d(x', A) \}.
\end{align*}
\end{definition}
\begin{theorem} \label{hausdorfferror}
There exists $K_l$ sufficiently large and $K_u$ sufficiently small such that the following holds with probability at least $1 - \delta$.
For each $C \in \mathcal{C}_\lambda$, there exists $\widehat{C} \in \widehat{\mathcal{C}_\lambda}$ such that
\begin{align*}
d_{\text{Haus}}(C, \widehat{C}) \le 2\cdot (4\lambda / \check{C}_\beta )^{1/\beta} \cdot C_{\delta, n}^{2/\beta} \cdot k^{-1/2\beta}.
\end{align*}
\end{theorem}

\begin{proof}
For $K_l$ and $K_u$ appropriately chosen, we have Lemma~\ref{separation} and Lemma~\ref{connectedness} hold. Thus we have for $C \in \mathcal{C}_\lambda$, there exists  $\widehat{C} \in \widehat{\mathcal{C}_\lambda}$ such that 
\begin{align*}
C^{\oplus r_n(C)} \cap X \subseteq
\widehat{C} \subseteq \bigcup_{\substack{ x\in \mathcal{X}_C \cap X\\  f_k(x) \ge \lambda - \frac{C_{\delta, n}^2}{\sqrt{k}} \lambda}} B(x, \varepsilon) \cap M.
\end{align*}
Define $\bar{r} := \left(\frac{4\lambda \cdot C_{\delta, n}^2}{\check{C}_\beta\cdot \sqrt{k}}\right)^{1/\beta}$. We show that $d_{\text{Haus}}(C, \widehat{C}) \le \bar{r}$, which involves two directions to show from the Hausdroff metric: that $\max_{x \in \widehat{C}} d(x, C) \le \bar{r}$ and $\sup_{x \in C} d(x, \widehat{C}) \le \bar{r}$. 

We start by proving $\max_{x \in \widehat{C}} d (x, C) \le \bar{r}$.
Define $r_0 = \bar{r}/2$. We have
\begin{align*}
r_0 = \frac{1}{2} \left(\frac{4\cdot C_{\delta, n}^2}{\check{C}_\beta \cdot \sqrt{k}} \right)^{1/\beta} \ge \left(\frac{k}{v_d n \lambda_0}\right)^{1/d} \ge \varepsilon,
\end{align*}
where the first inequality holds when $K_u$ is chosen sufficiently small, and 
the last inequality holds because $\lambda_0 < \lambda - \frac{C_{\delta, n}^2}{\sqrt{k}} \lambda$.
Hence $r_0 + \varepsilon \le \bar{r}$. Therefore, it suffices to show
\begin{align*}
\sup_{ x \in (\mathcal{X}_C \backslash C^{\oplus r_0}) \cap X} f_k(x) < \lambda - \frac{C_{\delta, n}^2}{\sqrt{k}} \lambda.
\end{align*}
We have that for $x \in (\mathcal{X}_C \backslash C^{\oplus r_0/2}) \cap X$, $f(x) \le \lambda - \check{C}_\beta(r_0/2)^\beta := \lambda'$.
Thus, for any $x \in (\mathcal{X}_C \backslash C^{\oplus r_0}) \cap X$ and letting $\epsilon = \lambda' - f(x)$, we have 
\begin{align*}
\hat{r}(\epsilon, x) &\ge r_0 / 2
\ge (4\lambda_0 C_{\delta, n}/(\sqrt{k} \cdot \check{C}_{\beta}))^{1/\beta} / 2.
\end{align*} 
For $K_u$ chosen sufficiently small, the last 
equation will be large enough (i.e. of order $(k/v_d n \lambda)^{1/d}$) so that the conditions 
of Lemma~\ref{fk_upper_bound} hold. Thus, applying this for each $x \in (\mathcal{X}_C \backslash C^{\oplus r_0}) \cap X$, we obtain
\begin{align*}
\sup_{x \in (\mathcal{X}_C \backslash C^{\oplus r_0}) \cap X} f_k(x) < \left(1 + 3 \frac{C_{\delta, n}}{\sqrt{k}} \right) (\lambda - \check{C}_\beta (r_0/2)^\beta).
\end{align*}
We have the r.h.s. is at most $\lambda - \frac{C_{\delta, n}^2}{\sqrt{k}} \lambda$ for $K_u$ chosen appropriately
and the first direction follows.

We now turn to the other direction, that $\sup_{x \in C} d (x, \widehat{C}) \le \bar{r}$. Let $x \in C$. Then there exists sample point $x' \in B(x, r_n(C))$ by definition of $r_n$ and we have that $x' \in \widehat{C}$. Finally, $r_n(C) \le \bar{r}$ for $K_l$ sufficiently large, and thus $|x' - x| \le \bar{r}$. The result follows.
\end{proof}

\begin{remark}
When taking $k \approx n^{2\beta'/(2\beta' + d)}$, we obtain the error rate of $d_{\text{Haus}}(C, \widehat{C}) \approx n^{-1/(2\beta + d\cdot \max\{1, \beta\})}$, ignoring logarithmic factors. 
When $0 < \beta \le 1$, this matches the known lower bound established in Theorem 4 of \citet{tysbakovMinimax}. 
However, we do not obtain this rate when $\beta > 1$. In this case, the density estimation error will be of order at least $n^{-1/(2+d)}$ due in part to 
the error from resolving the geodesic balls with Euclidean balls. This does not arise in the full dimensional setting, which will be described later. 
\end{remark}

\section{Removal of False Clusters} \label{pruning-section}

The result of Theorem~\ref{hausdorfferror} guarantees us that for each $C \in \mathcal{C}_\lambda$, there 
exists $\widehat{C} \in \widehat{\mathcal{C}}_\lambda$ that estimates it. 
In this section, we show how a second application of DBSCAN (Algorithm~\ref{alg:dbscan-pruning}) can remove the
false clusters discovered by the first application of DBSCAN with no additional parameters.
This gives us the other direction, that each estimate in $\widehat{\mathcal{C}}_\lambda$
corresponds to a true CC in $\mathcal{C}_\lambda$, and thus DBSCAN can identify with a one-to-one correspondence each CC of the level-set.

\begin{algorithm}[tbh]
   \caption{DBSCAN False CC Removal}
   \label{alg:dbscan-pruning}
\begin{algorithmic}
	\STATE As in Section~\ref{parameters}, let $\text{minPts} = k$ and
	\STATE $\varepsilon = \left(\frac{k}{n\cdot v_d \cdot (\lambda - \lambda \cdot C_{\delta, n}^2 /\sqrt{k})} \right)^{1/d}$.
   \STATE Define $\tilde{\varepsilon} := \left(\frac{k}{n\cdot v_d \cdot (\lambda - \lambda \cdot C_{\delta, n}^2 /\sqrt[3]{k})} \right)^{1/d}$.

	\STATE Let $\widehat{\mathcal{C}}_{\lambda}$ be the clusters returned by DBSCAN(minPts, $\varepsilon$).
	\STATE Let $\widehat{\mathcal{D}}_{\lambda}$ be the clusters returned by DBSCAN(minPts, $\tilde{\varepsilon}$).
	\STATE Let $\widetilde{\mathcal{C}}_\lambda$ be the clusters obtained by merging clusters from $\widehat{C}_{\lambda}$ which are subsets of the same 
	cluster in $\widehat{\mathcal{D}}_{\lambda}$ .
	\STATE {\bf Return} $\widetilde{\mathcal{C}}_\lambda.$
\end{algorithmic}
\end{algorithm}

We state our result below. The proof is less involved and is in the appendix.
\begin{theorem} [Removal of False CC Estimates] \label{pruning}
Define $\gamma = \lambda - \sup_{x \in M \backslash (\cup_{C \in \mathcal{C}_\lambda} \mathcal{X}_C)}  f(x)$, which is positive.
There exists $K_l$ sufficiently large and $K_u$ sufficiently small depending on 
$\gamma$ in addition to the constants mentioned in Section~\ref{parameters} so that the following
holds with probability at least $1 - \delta$. For all $\widehat{C} \in \widetilde{\mathcal{C}}_\lambda$, there exists $C \in \mathcal{C}_\lambda$
such that
\begin{align*}
d_{\text{Haus}}(C, \widehat{C}) \le 2\cdot (4\lambda / \check{C}_\beta )^{1/\beta} \cdot C_{\delta, n}^{2/\beta} \cdot k^{-1/2\beta}.
\end{align*}
\end{theorem}

\section{Adaptive Parameter Tuning} \label{adaptive-section}

 In this section, we show how to obtain the near optimal rates by 
 estimating $d$ and adaptively choosing $k$ such that $k \approx n^{2\beta'/(2\beta' + d)}$ without
 knowledge of $\beta$.

\subsection{Determining $d$}
Knowing the manifold dimension $d$ is necessary to tune the parameters
as described in Section~\ref{parameters}. There has been much work done on estimating the 
intrinsic dimension as many learning procedures (including this one) require $d$ as an input.
Such work in intrinsic dimension estimation include
\citep{intrinsicKegl,intrinsicBickel,intrinsicHein}.
\citet{intrinsicKnnOld} and more recently \citet{intrinsicKnnNew} take a $k$-nearest neighbor approach.
We work with the estimate of a dimension at a point proposed in the latter work:
\begin{align*}
\hat{d}(x) = \frac{\log 2}{\log (r_{2k}(x)/ r_k (x))}.
\end{align*}
The main result of \citet{intrinsicKnnNew} gives a high-probability bound for a single sample $X_1 \in X$.
Here we give a high-probability bound under 
more mild smoothness assumptions which hold uniformly for all samples above some density-level given our new knowledge of $k$-NN density estimation rates. This may
be of independent interest.

\begin{theorem} \label{estimating_d}
Suppose that $f$ is $\alpha$-H\"older continuous for some $0 < \alpha \le 1$.
Choose $\bar{\lambda}_0 > 0$ and $\delta > 0$. Then there exists constants $C_1, C_2$ depending on $\delta, C_\alpha, \alpha, \tau, d, \bar{\lambda}_0$ such that if $k$ satisfies
\begin{align*}
C_1 \cdot (\log n)^2 \le k \le C_2 \cdot n^{2\alpha / (2\alpha + d)},
\end{align*}
then with probability at least $1 - \delta$,
\begin{align*}
|\hat{d}(x) - d| \le 20d\cdot||f||_\infty\cdot \frac{C_{\delta, n}}{\sqrt{k}},
\end{align*}
uniformly for all $x \in X$ with $f_k(x) \ge \bar{\lambda}_0$.
\end{theorem}
\vspace{-0.3cm}
\begin{proof}
We have for $x \in X$ such that if $f_k(x) \ge \bar{\lambda}_0$, then $f(x) \ge \lambda_0 := \bar{\lambda}_0 / 2$ by Lemma~\ref{fk_upper_bound} for $C_1$ chosen appropriately large and $C_2$ chosen appropriately small.
\begin{align*}
\hat{d}(x) =  \frac{\log 2}{\log (r_{2k}(x)/ r_k (x))} =  \frac{d\log 2}{\log 2 + \log (f_{k}(x)/ f_{2k} (x))}.
\end{align*}
We now try to get a handle on $f_k(x) / f_{2k}(x)$ and show it is sufficiently close to $1$.
Applying Lemma~\ref{fk_upper_bound} and~\ref{fk_lower_bound} with $\epsilon = \frac{C_{\delta, n}}{\sqrt{k}} f(x)$ and $C_1$, $C_2$ appropriately chosen so that the conditions for the two Lemmas hold (remember that here we have $\hat{r}(\epsilon, x), \check{r}(\epsilon, x) \ge (\epsilon / C_{\alpha})^{1/\alpha}$), we obtain
\begin{align*}
\frac{f_k(x)}{f_{2k}(x)} &\ge \frac{(1 - 3C_{\delta, n}/\sqrt{k}) (1 - C_{\delta, n}/\sqrt{k})  \cdot f(x)}{(1 + 3C_{\delta, n}/\sqrt{k}) (1 + C_{\delta, n}/\sqrt{k}) \cdot f(x)}\\
&\ge 1 - 9 \cdot \frac{C_{\delta, n}}{\sqrt{k}},
\end{align*}
where the last inequality holds when $C_1$ is chosen sufficiently large so that $C_{\delta, n}/\sqrt{k}$ is sufficiently small. On the other hand, we similarly obtain (for $C_1$ and $C_2$ appropriately chosen):
\begin{align*}
\frac{f_k(x)}{f_{2k}(x)} &\le \frac{(1 + 3C_{\delta, n}/\sqrt{k}) (1 + C_{\delta, n}/\sqrt{k}) \cdot f(x)}{(1 - 3C_{\delta, n}/\sqrt{k})(1 - C_{\delta, n}/\sqrt{k})\cdot f(x)}\\
&\le  1 + 9 \cdot \frac{C_{\delta, n}}{\sqrt{k}}.
\end{align*}
It is now clear that by the expansion $\log(1 - r) = - r - r^2/2 - r^3 /3 - \cdots$, and for $K_l$ chosen sufficently large so that $C_{\delta, n}/\sqrt{k}$ is sufficiently small, we have
\begin{align*}
\left| \log \left(\frac{f_{k}(x)}{ f_{2k} (x)} \right) \right| \le 10 \cdot \frac{C_{\delta, n}}{\sqrt{k}}.
\end{align*}
The result now follows by combining this with the earlier established expression for $\hat{d}(x)$, as desired.
\end{proof}

\begin{remark}
In \citet{intrinsicKnnNew}, it is the case that $\alpha = 1$; under this setting, we match their bound with an error rate of $n^{1/(2+d)}$ with $k \approx n^{2/(2+d)}$ being the 
optimal choice for $k$ (ignoring log factors).   
\end{remark}

\subsection{Determining $k$}

After determining $d$, the next parameter we look at is $k$. In particular, to obtain the optimal rate, we
must choose $k\approx n^{2\beta'/(2\beta' + d)}$ without knowledge of $\beta$. We present a consistent
estimator for $\beta$.

We need the following definition. The first characterizes how much $f$ varies in balls of a certain radius along
the boundaries of the $\lambda$-level set (where $\partial \mathcal{C}_\lambda$ denotes the boundary of $\mathcal{C}_\lambda$). The second is meant to be an estimate of the first, which can be computed from the data alone.
The final is our estimate of $\beta$. 
\begin{align*}
D_{r} &= \inf_{x_0 \in \partial \mathcal{C}_\lambda} \sup_{x \in B(x_0, r)} |\lambda - f(x)|\\
\hat{D}_{r, k} &= \min_{\substack{x_0 \in X \\ B(x_0, r) \cap X \neq \emptyset}} \max_{x \in B(x_0, r)\cap X} |\lambda - f_k(x)| \\
\hat{\beta} &= \log_r ( \hat{D}_{r, k})
\end{align*}
The next is a result of how $\hat{D}_{r, k}$ estimates $D_r$. 

\begin{lemma} \label{alpha_bound} Suppose that $f$ is $\alpha$-H\"older continuous for some $0 < \alpha \le 1$. 
Let $k = \lfloor (\log n)^5 \rfloor$ and $r = 1 / \sqrt{\log n}$. Then there exists positive constants $\tilde{C}$ and $N$ depending on $d, \tau, \alpha, C_\alpha, \lambda_0, ||f||_\infty, r_c$ 
such that when $n \ge N$, then the following holds with probability at least $1 - 1/n$.
\vspace{-0.1cm}
\begin{align*}
|D_r - \hat{D}_{r, k}| \le \tilde{C} / (\log n)^2.
\end{align*}
\end{lemma}
\vspace{-0.5cm}
\begin{proof}[Proof sketch]
Suppose that the value of $D_r$ is attained at $x_0 = p$ and the value of $\hat{D}_{r, k}$ is attained at $x_0 = q$.
Let $y, z$ be the points that maximize $|\lambda - f(x)|$ on $B(p, r)$ and $B(q, r)$, respectively. Let $\hat{y}, \hat{z}$ be the sample points that maximize $|\lambda - f_k(x)|$ on 
$B(p, r)$ and $B(q, r)$, respectively. Now, we have
\vspace{-0.2cm}
\begin{align*}
&D_r - \hat{D}_{r, k} = |\lambda - f(y)| - |\lambda - f_k(\hat{z})|\\
&\le |\lambda - f(z)| - |\lambda - f_k(\hat{z})|  \le |f(z) - f_k(\hat{z})| \\
&\le \max \{ f(z) - f_k(z), f_k(\hat{z}) - f(\hat{z}) \}.
\end{align*}
Now let $z'$ be the closest sample point to $z$ in $B(q, r)$. Then,
\begin{align*}
&\le \max \{ f(z') - f_k(z'), f_k(\hat{z}) - f(\hat{z}) \} + |f(z) - f(z')|  \\
&+ |f_k(z) - f_k(z')| \le \max_{x \in X, f(x) \ge \lambda_0} |f(x) - f_k(x)|  \\
&+ C_\alpha|z - z'|^\alpha+ |f_k(z) - f_k(z')|.
\end{align*}
\vspace{-0.2cm}
On the other hand, we have
\begin{align*}
&\hat{D}_{r, k} - D_r = |\lambda - f_k(\hat{z})| - |\lambda - f(y)| \\
&\le |\lambda - f_k(\hat{y})| - |\lambda - f(y)| \le |f(y) - f_k(\hat{y})| \\
&\le \max \{ f(y) - f_k(y), f_k(\hat{y}) - f(\hat{y}) \}.
\end{align*}
\vspace{-0.2cm}
Let $y'$ be the closest sample point to $y$ in $B(p, r)$. Then,
\begin{align*}
&\le \max \{ f(y') - f_k(y'), f_k(\hat{y}) - f(\hat{y}) \} + |f(y) - f(y')|\\
& + |f_k(y) - f_k(y')| \le \max_{x \in X, f(x) \ge \lambda_0} |f(x) - f_k(x)| \\
&+  C_\alpha |y-y'|^\alpha + |f_k(y) - f_k(y')|.
\end{align*}
Thus it suffices to bound $\max_{x \in X, f(x) \ge \lambda_0} |f(x) - f_k(x)|, |y-y'|, |z-z'|,
 |f_k(y) - f_k(y')|,  |f_k(z) - f_k(z')|$. First take $\delta = 1/n$ and use 
 Lemma~\ref{fk_upper_bound} and~\ref{fk_lower_bound} for $\max_{x \in X, f(x) \ge \lambda_0} |f(x) - f_k(x)|$.
 Using Lemma~\ref{ballbounds}, we can show that $r_n := |y - y'| \lesssim (\log n/n)^{1/d}$.
 Next we bound $|f_k(y) - f_k(y')|$. $y' \in X$ so we have guarantees on its $f_k$ value.
 Note that $r_k(y') - r_n \le r_k(y) \le r_k(y')+ r_n$. Let $r_k = r_k(y')$. This implies that
$ f_k(y') (r_k / (r_k + r_n))^d \le f_k(y) \le f_k(y') (r_k / (r_k - r_n))^d$.
Now since $r_k \approx (k / n)^{1/d}$, we have $|f_k(y) - f_k(y')| \lesssim \log n / k$.
The same holds for the bounds related to $z, z'$.
\end{proof}

\begin{theorem} [$\hat{\beta} \rightarrow \beta$ in probability] \label{estimating_alpha} Suppose 
$f$ is $\alpha$-H\"older continuous for some $\alpha$ with $0 < \alpha \le \beta'$.
Let $k = \lfloor (\log n)^5 \rfloor$ and $r = 1 / \sqrt{\log n}$. Then for all $\epsilon > 0$,
\vspace{-0.3cm}
\begin{align*}
\lim_{n \rightarrow \infty} \mathbb{P}\left(|\hat{\beta} - \beta| \ge \epsilon\right) = 0.
\end{align*}
\end{theorem}
\vspace{-0.5cm}
\begin{proof}
Based on the $\beta$-regularity assumption, we have for $r < r_c$:
\vspace{-0.5cm}
\begin{align*}
\check{C}_\beta r^\beta \le D_r \le \hat{C}_\beta r^\beta.
\end{align*}
Combining this with Lemma~\ref{alpha_bound}, we have with probability at least $1 - 1/\sqrt{n}$ that
\begin{align*}
\check{C}_\beta r^\beta  - \tilde{C} / (\log n)^2  \le \hat{D}_{r, k} \le \hat{C}_\beta r^\beta + \tilde{C} / (\log n)^2.
\end{align*}
Thus with probability at least $1 - 1/n$,
\begin{align*}
\beta - \hat{\beta} &\ge \frac{\log (1 -  \tilde{C}/(\hat{D}_{r, k} \cdot (\log n^2))) }{\log r} - \frac{\log \hat{C}_\beta}{\log r} \\
\beta - \hat{\beta} &\le \frac{\log (1 +  \tilde{C}/(\hat{D}_{r, k} \cdot (\log n^2))) }{\log r} + \frac{\log \check{C}_\beta}{\log r}.
\end{align*}
It is clear that these expressions go to $0$ as $n\rightarrow \infty$ and the result follows.
\end{proof}

\begin{remark}\label{adaptive_parameter}
We can then take $k = n^{\hat{\beta'}/(2\hat{\beta'} + d)}$ with $\hat{\beta'} = \min\{1, \hat{\beta} - \epsilon_0\}$ for some $\epsilon_0 > 0$ so that $\hat{\beta'} < \beta'$ for $n$ sufficiently large and thus $k$ lies in the allowed ranges described in Section~\ref{parameters} asymptotically.
The settings of $\varepsilon$ and $\text{MinPts}$ are implied by this choice of $k$ and our estimate of $d$.
\end{remark}
\subsection{Rates with Data-driven Tuning}
Putting this all together, along with Theorems~\ref{hausdorfferror} and~\ref{pruning}, gives us the following consequence about level set recovery with adaptive tuning. It shows that we can obtain rates arbitrarily close to those obtained as if the smoothness parameter $\beta$ and intrinsic dimension were known.
\begin{corollary} \label{nearoptimal}
Suppose that $0 < \delta < 1$ and $f$ is $\alpha$-H\"older continuous for some $0 < \alpha \le 1$ and suppose
the data-driven choices of parameters described in Remark~\ref{adaptive_parameter} are used for DBSCAN.
For any $\epsilon > 0$, there exists $N_{\epsilon, \delta, f} \equiv N(\epsilon, \delta, f)$ and $C_\delta \equiv C_\delta(\delta, f)$ such that the following holds. If $n \ge N_{\epsilon, \delta, f}$, then with probability at least $1 - \delta$ simulatenously for each $C \in \mathcal{C}_\lambda$, there exists $\widehat{C} \in \widehat{\mathcal{C}_\lambda}$ such that
\begin{align*}
d_{\text{Haus}}(C, \widehat{C}) \le C_{\delta} \cdot n^{-\frac{1}{2\beta + d\max \{1, \beta\}}+ \epsilon}.
\end{align*}
Moreover, using Algorithm~\ref{pruning}, there is a one-to-one correspondence between $\mathcal{C}_\lambda$ and $\widehat{\mathcal{C}_\lambda}$.
\end{corollary}

\section{Full Dimensional Setting} \label{fulldimensionalsetting-section}

Here we instead take $f$ to be the density of $\mathcal{F}$ over the uniform measure on $\mathbb{R}^D$. Let
\begin{align*}
&\text{minPts} = k, \text{   } \varepsilon = \left(\frac{k}{n\cdot v_D \cdot (\lambda - \lambda \cdot C_{\delta, n}^2/\sqrt{k})} \right)^{1/D},
\end{align*}
where $k$ satisfies
\begin{align*}
K_l \cdot (\log n)^2 \le k \le K_u \cdot (\log n)^{2D/(2+D)} \cdot n^{2\beta / (2\beta + D)},
\end{align*}
and $K_l$ and $K_u$ are positive constants depending $\delta, \check{C}_\beta, \hat{C}_\beta, \beta, \tau, D, ||f||_\infty, \lambda_0, r_s, r_c$.

Then Theorem~\ref{hausdorfferror} and~\ref{pruning} hold (replacing $d$ with $D$ in Algorithm~\ref{alg:dbscan-pruning}) for this setting of DBSCAN and thus taking $k \approx n^{2\beta/(2\beta + D)}$ gives us the optimal 
estimation rate of $O(n^{-1/(2\beta + D)})$. A straightforward modification of Corollary~\ref{nearoptimal} also holds. This is discussed further in the Appendix.

\section{Conclusion}
We proved that DBSCAN can obtain Hausdorff level-set recovery rates of
$\widetilde{O}(n^{-1/(2\beta + D)})$ when the data is in $\mathbb{R}^D$, and
$\widetilde{O}(n^{-1/(2\beta + d\cdot \max\{1, \beta \})})$ when the data lies on an embedded $d$-dimensional manifold.
The former rate is optimal up to log factors and the latter matches known $d$-dimensional lower bounds for $0 < \beta \le 1$ up to log factors.
Moreover, we provided a fully data-driven procedure to tune the parameters to attain these rates.  

This shows that the procedure's ability to recover density level sets matches
the strongest known consistency results
attained for this problem.
Furthermore, we developed the necessary tools and give the first analysis
of density level-set estimation on manifolds, let alone with a practical procedure such as DBSCAN.

Our density estimation errors however cannot converge faster than $\widetilde{O}(n^{-1/(2+d)})$,
which is due in part to the error from resolving geodesic balls with Euclidean balls. Thus it remains an open problem
whether the manifold level-set rates are minimax optimal when $\beta > 1$.

\section*{Acknowledgements}
The author is grateful to Samory Kpotufe for insightful discussions and to the anonymous reviewers for their useful feedback.

\bibliography{paper}
\bibliographystyle{icml2017}

\newpage
{\onecolumn
{\Large \bf Appendix}
\appendix
\section{Proof of Ball Volume Result}

\begin{proof}[Proof of Lemma~\ref{ballvolume}]
For the lower bound, we have by Lemma 5.3 of \cite{lowerBoundBall} that
\begin{align*}
\text{vol}_d(B(x, r) \cap M) \ge v_d\cdot r^d\cdot \cos\left(\arcsin\left(\frac{r}{2\tau}\right)\right) \ge v_d\cdot r^d\cdot (1 - r^2\tau^2),
\end{align*}
where the last inequality holds for $r < 1/\tau$. For the upper bound, we have by \cite{upperBoundBall} that
\begin{align*}
\text{vol}_d(B(x, r) \cap M)  \le v_d\cdot r^d\cdot \left( \frac{1}{2\sqrt{1 - 2r/\tau} - 1}\right)^d \le v_d\cdot r^d\cdot (1 + 4dr/\tau),
\end{align*}
where the last inequality holds when $r < \tau/4d$.
\end{proof}

\section{Proof of Density Estimation Results}

\begin{proof}[Proof of Lemma~\ref{fk_upper_bound}]
Choose $r$ such that
\begin{align*}
(1 + 4dr/\tau) \cdot v_d \cdot r^d \cdot (f(x) + \epsilon)  
= \frac{k}{n} - C_{\delta, n} \frac{\sqrt{k}}{n} 
\le v_d \cdot \hat{r}(\epsilon, x)^d \cdot (f(x) + \epsilon).
\end{align*}
Thus, $r \le \hat{r}(\epsilon, x)$ and hence
\begin{align*}
\mathcal{F}(B(x, r)) \le \text{vol}_d(B(x, r) \cap M) \cdot (f(x) + \epsilon) 
\le (1 + 4dr/\tau) \cdot v_d \cdot r^d \cdot (f(x) + \epsilon)
= \frac{k}{n} - C_{\delta, n} \frac{\sqrt{k}}{n},
\end{align*}
where the second inequality follows by Lemma~\ref{ballvolume}.
Thus, by Lemma~\ref{ballbounds}, we have $\mathcal{F}_n(B(x, r)) < k/n$ and hence $r_k(x) > r$. Therefore, 
\begin{align*}
f_k(x) < \frac{k}{nv_dr^d} = \frac{1 +  4dr/\tau}{1 - C_{\delta, n}/\sqrt{k}} (f(x) + \epsilon) 
\le \left( 1 + 2\frac{C_{\delta, n}}{\sqrt{k}} + 4dr/\tau \right) (f(x) + \epsilon)
\le \left( 1 + 3\frac{C_{\delta, n}}{\sqrt{k}} \right) (f(x) + \epsilon).
\end{align*}
The last inequality holds because $r \le (k/(n\cdot v_d \cdot (f(x) + \epsilon)))^{1/d}$ and choosing $C_1$ appropriately we have $4rd/\tau \le 
4d (k/(n\cdot v_d \cdot \lambda_0))^{1/d} /\tau < C_{\delta, n}/\sqrt{k}$. 
\end{proof}

\begin{proof}[Proof of Lemma~\ref{fk_lower_bound}]
Let $x \in X$ such that $f(x) \ge \lambda_0$. Choose $r$ such that
\begin{align*}
(1 - \tau^2 r^2) \cdot v_d \cdot r^d \cdot (f(x) - \epsilon) = \frac{k}{n} + C_{\delta, n} \frac{\sqrt{k}}{n}.
\end{align*}
We have $r \le c (2k/(n\cdot v_d \cdot (f(x) - \epsilon)))^{1/d}$ for some constant $c$ depending on $d$ and $\tau$.
Thus, choosing $C_2$ approrpiately, we have $\tau^2r^2 < \frac{1}{4}$. Thus,
\begin{align*}
v_d \cdot r^d \cdot (f(x) - \epsilon) \le  \frac{4}{3} \left( \frac{k}{n} + C_{\delta, n} \frac{\sqrt{k}}{n} \right) \le v_d \cdot \check{r}(\epsilon, x)^d \cdot (f(x) - \epsilon).
\end{align*}
Thus $r <  \check{r}(\epsilon, x)$ and we obtain:
\begin{align*}
\mathcal{F}(B(x, r)) \ge \text{vol}_d(B(x, r) \cap M) \cdot (f(x) - \epsilon) 
\le (1 - \tau^2 r^2) \cdot v_d \cdot r^d \cdot (f(x) - \epsilon)
= \frac{k}{n} + C_{\delta, n} \frac{\sqrt{k}}{n}.
\end{align*}
Thus, by Lemma~\ref{ballbounds}, we have $\mathcal{F}_n(B(x, r)) \ge k/n$ and hence $r_k(x) \le r$.
Therefore, 
\begin{align*}
f_k(x) &\ge \frac{k}{n\cdot v_d\cdot r^d} = \frac{1 - \tau^2 r^2}{ 1 + C_{\delta, n}/\sqrt{k}} (f(x) - \epsilon) \ge (1 - \tau^2 r^2 - C_{\delta,n}/\sqrt{k}) \cdot (f(x) - \epsilon) \ge \left(1 - 3 \cdot \frac{C_{\delta, n}}{\sqrt{k}} \right)\cdot (f(x) - \epsilon),
\end{align*}
where the last inequality follows $\tau^2r^2 \le \tau^2 c^2 (2k/(n\cdot v_d \cdot \lambda_0))^{2/d} \le \frac{C_{\delta, n}}{\sqrt{k}}$ where
the latter holds for $C_2$ chosen appropriately. The result follows.
\end{proof}

\section{Proof of Separation Result}

\begin{proof} [Proof of Lemma~\ref{separation}]
Define $\bar{r} := \min \{r_c, r_s, \frac{1}{2} \tau\}$. 
We show that $(M \backslash \mathcal{X}_C) \cap X$ and $C^{\oplus \bar{r}} \cap X$ are disconnected in $G(k, \varepsilon)$.
To do this, it suffices to show that:
\begin{itemize}
\item (1) $G(k, \varepsilon)$ has no point in $\mathcal{X}_C \backslash C^{\oplus \bar{r}}$.
\item (2) $G(k, \varepsilon)$ has no point in $S_C^{\oplus r_s/2}$.
\item (3) $G(k, \varepsilon)$ has no edge connecting a point in $C^{\oplus \bar{r}}$ to a point in $\mathcal{X}_C \backslash C^{\oplus \bar{r}}$.
\end{itemize}
We begin by showing (1). Define $\lambda' := \lambda - \check{C}_\beta(\bar{r}/2)^\beta$. Thus, for $x \in C^{\oplus \bar{r}} \cap X$, we have $\hat{r}(\lambda' - f(x), x) \ge \bar{r}/2$. Thus the conditions for Lemma~\ref{fk_upper_bound} are satisfied as long as $K_l$ and $K_u$ are appropriately large and small, respectively. Hence,
\begin{align*}
\sup_{x \in X \cap (\mathcal{X}_C \backslash C^{\oplus \bar{r}})} f_k(x) 
\le \left(1 + 3\cdot \frac{C_{\delta, n}}{\sqrt{k}}\right) (\lambda - \check{C}_\beta (\bar{r}/2)^\beta)
< \lambda - \frac{C_{\delta, n}^2}{\sqrt{k}} \cdot \lambda,
\end{align*}
where the last inequality holds for $K_l$ sufficiently large. Thus, it is now clear that $r_k(x) > \varepsilon$ for $x \in X \cap (\mathcal{X}_C \backslash C^{\oplus \bar{r}})$, showing (1). Next, if $x \in S_C^{\oplus r_s/2}$, then $\hat{r}(\lambda' - f(x), x) \ge \bar{r}/2$ and the same holds for sample points in $S_C^{\oplus r_s/2}$ implying (2). 

To show (3), it suffices to show that any such edge will have length less than $r_s$ since $S_C^{\oplus r_s/2}$ separates $C^{\oplus \bar{r}}$  and $\mathcal{X}_C \backslash C^{\oplus \bar{r}}$ by length at least $r_s$. Indeed, for $x \in X \cap C^{\oplus \bar{r}}$,
\begin{align*}
\mathcal{F}(B(x, \bar{r})) \ge \text{vol}_d(M \cap B(x, \bar{r})) \inf_{x' \in M \cap B(x, 2\bar{r})} f(x')
\ge (1 - \tau^2\bar{r}^2) \cdot v_d\cdot \bar{r}^d (f(x) - \hat{C}_\beta\bar{r}^\beta)
\ge \frac{k}{n} + C_{\delta, n}\frac{\sqrt{k}}{n},
\end{align*}
where the second inequality follows from Lemma~\ref{ballvolume}, and the last inequality holds when $K_u$ is sufficiently small.
By Lemma~\ref{ballbounds}, we have $r_k(x) \le \bar{r} < r_s$, establishing (3). Thus, $(M \backslash \mathcal{X}_C) \cap X$ and $C^{\oplus \bar{r}} \cap X$ are disconnected in $G(k, \varepsilon)$.

It is easy to see that $C^{\oplus \bar{r}}$ contains a sample point and thus there exists a connected component, $\overline{C}$
of $G(k, \varepsilon)$ such that $\overline{C} \subseteq \mathcal{X}_C$.
Now, by Lemma~\ref{dbscan}, we have that there exists $\widehat{C} \in \widehat{\mathcal{C}}_\lambda$ such that
$\widehat{C} = \{ x \in X : d(x, C) \le \varepsilon \}$. Since $\overline{C}$ has no intersection with $S_C^{\oplus r_s/2}$, it follows that as long as $\varepsilon < r_s/2$, then $\widehat{C} \subseteq \mathcal{X}_C$. For $K_u$ chosen sufficiently small, we have $\varepsilon < r_s/2$, as desired.
\end{proof}

\section{Proof of Connectedness}

\begin{proof} [Proof of Lemma~\ref{connectedness}]
Define $A := C^{\oplus r_n(C)}$. 
In light of Lemma~\ref{dbscan}, it suffices to show that $A \cap X$ is connected in $G(k, \varepsilon)$.

Define $r_o := (k/(2nv_d ||f||_\infty))^{1/d}$. We next show that for each $x \in A^{\oplus \varepsilon}$,
we have a sample point in $B(x, r_o)$. Indeed, for any $z \in B(x, r_0/3) \cap \mathcal{N}$ (while taking $K_l$ sufficiently large so that $r_0/3 > 1/n$):
\begin{align*}
\mathcal{F}(B(z, r_o/2)) &\ge \text{vol}_d(B(z, r_o/2) \cap M) \inf_{x' \in B(z, r_o/2 + \varepsilon )} f(x') \\
&\ge (1 - \tau^2 (r_o/2)^2) \cdot  v_d \cdot (r_o/2)^d \cdot (\lambda  - \hat{C}_\beta (r_o /2+ \varepsilon + r_n(C))^\beta)  \\
&\ge (1 - \tau^2 (r_o/2)^2) \cdot  v_d \cdot (r_o/2)^d \cdot (\lambda  - \hat{C}_\beta (3\cdot \varepsilon)^\beta)  \ge C_{\delta, n} \frac{\sqrt{d \log n}}{n},
\end{align*}
with the last inequality holding when $K_u$ is chosen sufficiently small so that the $\tau^2 (r_o/2)^2$ and $\hat{C}_\beta (3\cdot \varepsilon)^\beta$ terms
become small enough, and $K_l$ is chosen sufficiently large so that the $(r_o/2)^d$ factor will be large enough.
Thus by Lemma~\ref{ballbounds} we have that with probability at least $1 - \delta$, $B(z, r_o/2) \subseteq B(x, r_o)$ contains a sample. 

Now, let $x$ and $x'$ be two points in $A\cap X$. We show there exists a path $x = x_0,x_1,...,x_p = x'$ such that $||x_i - x_{i+1}|| < r_o$ and $x_i \in B(A, r_o)$. For arbitrary $\gamma \in (0, 1)$, we can choose $x = z_0, z_1,...,z_p = x'$ where $||z_{i+1} - z_i|| \le \gamma r_o$. Next, for $K_l$ sufficiently large, there exists $\gamma$ sufficiently small such that 
\begin{align*}
\left(1 - \tau^2 \frac{(1-\gamma)^2r_o^2}{4}\right) \cdot v_d\left( \frac{(1-\gamma)r_o}{2}\right)^d \inf_{z \in B(A, r_0)} f(z) \ge \frac{C_{\delta, n}\sqrt{d\log n}}{n}.
\end{align*}
Therefore by Lemma~\ref{ballbounds}, there exists a sample point $x_i$ in $B(z_i, (1-\gamma)r_o/2)$ and
\begin{align*}
||x_{i+1} - x_i|| \le ||x_{i+1} - z_{i+1}|| + ||z_{i+1} - z_i|| + ||z_i -x_i|| \le r_o.
\end{align*} 

\noindent All that remains is to show $(x_i, x_{i+1}) \in G(k, \varepsilon)$. 
We see that $x_i \in B(A, r_o)$ for each $i$ and for any $x \in B(A, r_o)$, we have 
\begin{align*}
\mathcal{F}(B(x, \varepsilon)) &\ge (1 - \tau^2 \varepsilon^2) \cdot v_d\cdot \varepsilon^d \inf_{x' \in B(x, r_o + \varepsilon)} f(x') 
\ge (1 - \tau^2 \varepsilon^2)  \cdot v_d \cdot \varepsilon^d \cdot (\lambda - \hat{C}_\beta (2r_o + \varepsilon)^\beta)
\ge \frac{k}{n} + \frac{C_{\delta, n} \sqrt{k}}{n},
\end{align*}
where the last inequality holding when $K_u$ is chosen sufficiently small so that the $\tau^2 \varepsilon^2$ and $\hat{C}_\beta (2r_o +\varepsilon)^\beta$
become small enough, and $K_l$ is chosen sufficiently large so that the $\varepsilon^d$ factor will be large enough.

 Therefore, $r_k(x_i) \le \varepsilon$ and so $x_i \in G(k, \varepsilon)$ for all $x_i$. Finally, $||x_{i+1} - x_i|| \le r_o \le \varepsilon$. Hence, $(x_i, x_{i+1}) \in G(k, \varepsilon)$. The result immediately follows.
\end{proof}

\section{Proof of Theorem~\ref{pruning}}

\begin{proof} [Proof Sketch]
Define $x \in \widehat{C}$ as a core-point if $r_k(x) \le \varepsilon$. For any core point $x \in \widehat{C}$, we have 
$f_k(x) \ge \lambda -  \frac{C_{\delta, n}^2}{\sqrt{k}}\lambda$. With $K_l, K_u$ chosen appropriately, we can apply Lemma~\ref{fk_upper_bound} with $\epsilon = \frac{C_{\delta, n}}{\sqrt{k}} \lambda$ to obtain
\begin{align*}
f(x) &\ge \frac{f_k(x)}{1 + 3 \cdot C_{\delta, n}/\sqrt{k}} - \epsilon 
\ge \left(1 - 5\cdot \frac{C_{\delta, n}}{\sqrt{k}}\right) \cdot \lambda.
\end{align*}
Defining $r := (5\lambda \cdot C_{\delta, n} / (\check{C}_\beta \sqrt{k}))^{1/\beta}$, we have that $x \in C^{\oplus r}$ for some $C \in \mathcal{C}_\lambda$ when $r < \min \{r_c, (\gamma/\hat{C}_\beta)^{1/\beta} \}$, which is achieved as long as $K_l$ is chosen sufficiently large.

Since any core point of $\widehat{C}$ is in $C^{\oplus r}$, then in 
light of Theorem~\ref{hausdorfferror} and Lemma~\ref{dbscan}, it suffices to show that
separate CCs of $\mathcal{C}_\lambda$ do not get merged in $G(k, \tilde\varepsilon)$ (separation) and that
$C^{\oplus r} \cap X$ appear in the same connected component of $G(k, \tilde\varepsilon)$ (connectedness).

These follow from Lemma~\ref{separation} and~\ref{connectedness},  but with the minor modification that $\varepsilon$ is replaced by $\tilde{\varepsilon}$, which requires $K_l$ and $K_u$ to be adjusted accordingly to hold. Otherwise the arguments are the same.
\end{proof}

\section{Full Dimensional Setting}

The analysis for this situation is largely the same as under the manifold assumption. We will only highlight the main difference, which is
in the density estimation bounds.
We can utilize such bounds from \cite{optimalknn}, which are repeated here.

\begin{lemma} [Lemma 3 of \cite{optimalknn}]
Suppose that $k \ge 4 C_{\delta, n}^2$. Then with probability at least $1-\delta$, the following holds for all $x\in \mathbb{R}^d$ and $\epsilon > 0$.
\begin{align*}
f_k(x) < \left(1 + 2 \frac{C_{\delta, n}}{\sqrt{k}} \right) (f(x) + \epsilon),
\end{align*}
provided $k$ satisfies $v_d \cdot \hat{r}(x, \epsilon) \cdot (f(x) + \epsilon) \ge \frac{k}{n} + C_{\delta, n}\frac{\sqrt{k}}{n}$.
\end{lemma}
\begin{lemma} [Lemma 4 of \cite{optimalknn}]
Suppose that $k \ge 4 C_{\delta, n}^2$. Then with probability at least $1-\delta$, the following holds for all $x\in \mathbb{R}^d$ and $\epsilon > 0$.
\begin{align*}
f_k(x) \ge \left(1 -\frac{C_{\delta, n}}{\sqrt{k}} \right) (f(x) - \epsilon),
\end{align*}
provided $k$ satisfies $v_d \cdot \check{r}(x, \epsilon) \cdot (f(x) - \epsilon) \ge \frac{k}{n} - C_{\delta, n}\frac{\sqrt{k}}{n}$.
\end{lemma}

Unlike Lemma~\ref{fk_upper_bound} and~\ref{fk_lower_bound}, these results don't require $k \lesssim n^{2/(2+d)}$ and $k \lesssim n^{4/(4+d)}$, respectively.
This allows us to take $k \lesssim n^{2\beta/(2\beta + d)}$ rather than $k \lesssim \min \{n^{2/(2 + d)}, n^{2\beta/(2\beta + d)} \} = n^{2\beta'/(2\beta' + d)}$ 
in the analysis. Otherwise, the analysis is nearly identical up to constant factors.

}

\end{document}